\documentclass[10pt, twocolumn, twoside]{IEEEtran}

\usepackage{xcolor,soul,framed} %,caption
\usepackage[colorlinks,citecolor=black, linkcolor=blue, urlcolor=blue]{hyperref}
\colorlet{}{yellow}
\usepackage[pdftex]{graphicx}
\graphicspath{{../pdf/}{../jpeg/}}
\DeclareGraphicsExtensions{.pdf,.jpeg,.png}
\usepackage[cmex10]{amsmath}
\usepackage{amssymb}
\usepackage{array}
\usepackage{cite}
\usepackage{mdwmath}
\usepackage{amsthm}
\usepackage{algorithm}
\usepackage{algorithmic}

\theoremstyle{remark}
\newtheorem{theorem}{\textbf{Theorem}}
\newtheorem{proposition}{\textbf{Proposition}}

\newtheorem{definition}{\textbf{Definition}}
\usepackage{mdwtab}
\usepackage{amsfonts}
\usepackage{eqparbox}
\usepackage{url}
\usepackage{subfigure}

\begin{document}
\bstctlcite{}
    \title{Minimize Control Inputs for Strong Structural Controllability Using Reinforcement Learning with Directed Graph Neural Network}
  \author{Mengbang Zou, Weisi Guo$^*$, Bailu Jin\\

\thanks{M. Zou is with Cranfield University, Cranfield, MK43 0AL, U.K. (e-mail:m.zou@cranfield.ac.uk)}
\thanks{W. Guo (corresponding author) is with Cranfield University, Cranfield MK43 0AL, U.K., and also with the Alan Turing
Institute, London, NW1 2DB, U.K. (e-mail: weisi.guo@cranfield.ac.uk).} 
\thanks{B. Jin is with Cranfield University, Cranfield, MK43 0AL, U.K. (e-mail:bailu.jin@cranfield.ac.uk)}}

\maketitle

\begin{abstract}
Strong structural controllability (SSC) guarantees networked system with linear-invariant dynamics controllable for all numerical realizations of parameters. Current research has established algebraic and graph-theoretic conditions of SSC for zero/nonzero or zero/nonzero/arbitrary structure. One relevant practical problem is how to fully control the system with the minimal number of input signals and identify which nodes must be imposed signals. Previous work shows that this optimization problem is NP-hard and it is difficult to find the solution. To solve this problem, we formulate the graph coloring process as a Markov decision process (MDP) according to the graph-theoretical condition of SSC for both zero/nonzero and zero/nonzero/arbitrary structure. We use Actor-critic method with Directed graph neural network which represents the color information of graph to optimize MDP. Our method is validated in a social influence network with real data and different complex network models.  We find that the number of input nodes is determined by the average degree of the network and the input nodes tend to select nodes with low in-degree and avoid high-degree nodes.

\end{abstract}

\begin{IEEEkeywords}
controllability; reinforcement learning; complex network; directed graph neural network
\end{IEEEkeywords}

\IEEEpeerreviewmaketitle

\section{Introduction}
\IEEEPARstart{C}{ontrollability} of complex networks with linear time-invariant (LTI) dynamics as an essential property to achieve a desired global behavior with a suitable choice of inputs has attracted a lot of attention in recent years.  The controllability can be verified by the Kalman rank condition \cite{kalman1963mathematical}. However, for many complex networks, the system parameters are not precisely known. We only know whether there exists a link or not, but are not able to measure the weights of the links. Hence, it is difficult to numerically verify Kalman's controllability rank condition. Besides, even if all weights are known, Kalman's controllability rank condition is quite sensitive to the weights in a large complex network. Even slight perturbations in the weights can cause the controllability of the complex network totally different. To bypass the need of exact value of system parameters,  \cite{liu2011controllability} proposed a method by combining tools from control theory and network science to analyze the structure controllability of complex networks, where elements in system matrix are either fixed zeros or independent free parameters. This weak structural controllable system can be shown to be controllable
for almost all weight combinations, except for some pathological cases \cite{liu2011controllability}. However, some systems have interdependent parameters, making them uncontrollable despite the fact that it is structurally controllable. This leads to the notion of strong structural controllability (SSC): A system is strongly structurally controllable if it remains controllable for any value \cite{mayeda1979strong}. An algebraic condition for strong structural controllability was provided in \cite{reinschke1992strong}. \cite{trefois2015zero} provided a necessary and sufficient condition in terms of zero forcing sets for strong controllability. Most of the existing literature has considered strong structural controllablity under the assumption that elements in system matrix are either a fixed zero or an arbitrary nonzero value. However, in many scenarios, a third possibility exists, in which a given element is not a fixed zero or nonzero, but can take any real value. In such a scenario, it is not possible to describe the system by a zero/nonzero structure. To solve this problem, \cite{jia2020unifying} extended the zero/nonzero structure to a more general zero/nonzero/arbitrary structure and has established necessary and sufficient conditions for strong structural controllability under this zero/nonzero/arbitrary structure. 

Any networked system with LTI dynamics is fully controllable if independent signals are imposed on each node individually. But it is costly and impractical for large complex systems. Therefore, how to fully control the whole system with a minimum number of input signals is of interest. Minimum input for weak structural controllability is well-studied in \cite{olshevsky2015minimum, pequito2015framework, moothedath2018flow}. For minimum input problems of strong structural controllability, \cite{chapman2013strong} proved this problem is NP-complete and proposed a heuristic algorithm for small networks. \cite{pequito2013framework} proposed an algorithm to solve minimal SSC problems if the structured state matrix has a so-called maximal staircase structure. \cite{trefois2015zero} proved that the minimal SSC problem of a directed graph allowing loops is NP-hard and proposed a method to find the minimum size input set of SSC of a self-damped system with a tree structure. \cite{yashashwi2019minimizing} proposed a randomized algorithm based on Markov Chain Monte Carlo to find the minimum input set for general graph topology. All of these researches are based on the zero/nonzero structure. But the research on the minimum input problem of SSC for a more general zero/nonzero/arbitrary structure still remains elusive. 

\begin{figure*}[ht]
    \centering
    \resizebox*{15cm}{!}{\includegraphics{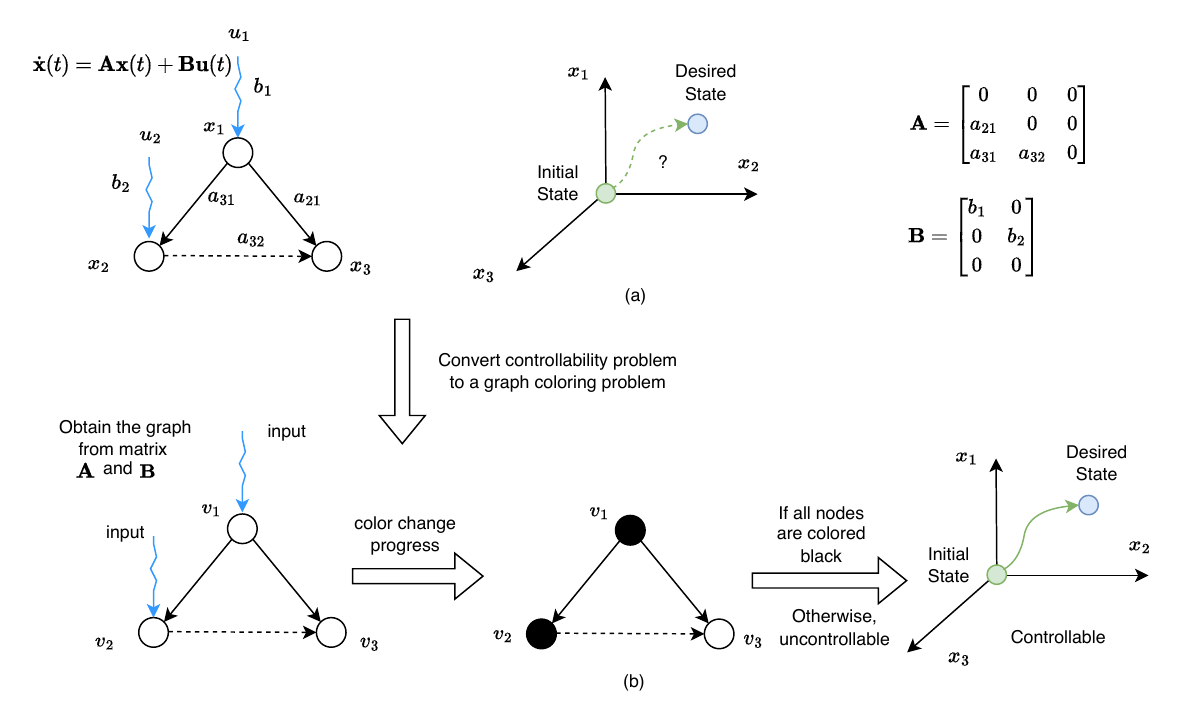}}
    \caption{This figure shows how to convert the problem of controllability to a graph coloring problem. (a) The networked system is controlled by input signals $u_1$ and $u_2$, allowing the system to achieve a desired state from the initial state. Matrix $\textbf{A}$ describes the connections among nodes. Matrix $\textbf{B}$ shows nodes imposed signals. (b) Whether the system is fully controlled by the imposed signals $u_1, u_2$ can be converted to a graph coloring problem. Nodes $v_1, v_2, v_3$ in (b) represent variables $x_1, x_2, x_3$ in (a). If all nodes of the graph obtained by matrix $\textbf{A}$ can be colored black according to the color change rule in Definition~(\ref{def: color}), then the system is controllable. Otherwise, the system is not controllable.}
    \label{fig:1}
\end{figure*}

The necessary and sufficient graph-theoretic condition for SSC with zero/nonzero/arbitrary structure has been established in \cite{jia2020unifying} and they have proved that the system is fully controllable if and only if the corresponding graph is colorable by a specific color change rule. Therefore, it is reasonable for us to convert the minimum control input problem of SSC to a graph coloring problem. By doing this, we can find minimum control input for SSC by coloring all nodes with the minimum number of input nodes. We can design a simple algorithm based on the color change rule and degree distribution of nodes to find the minimum number of input nodes. This method is effective in small networks.

Recently, reinforcement learning (RL) methods have been applied in combinatorial optimization on graphs, e.g. traveling salesman problem \cite{bello2016neural, nazari2018reinforcement, cappart2021combining}, maximum cut problem \cite{khalil2017learning, barrett2020exploratory, gu2020deep}, bin packing problem \cite{duan2019multi, laterre2018ranked, que2023solving}, minimum vertex cover problem \cite{song2020co, manchanda2020gcomb}, etc. Drawing inspiration from the successful application of RL in the graph optimization problem, we explore the potential of RL as a solution for the minimum control input problem of SSC. To solve the minimum control input problem of SSC, we transform it into a graph coloring optimization problem and formulate the process of graph coloring as a Markov Decision Process(MDP). Initially, all nodes are colored white, then the subsequent action is coloring a selected node into black. The introduction of the black node will force some white nodes to transition to black under specific color change rules, thereby advancing the state. The reward function is designed according to the number of black nodes in each state. It is important to note that the action and the state are closely related to the structure of the graph. However, neural networks in reinforcement learning methods such as deep Q learning and actor-critic networks are difficult to apply on the graph-based data with large discrete action space. Thus, we need a method which can represent the graph structure and generalize across similar states as well as actions. It is natural for us to consider using directed graph neural networks, which contains graph structure information, as actor-network and critic network.

The contribution of this paper is as followings. Firstly, we propose a framework to optimize the number of control inputs to make the system strong structural controllable for zero/nonzero/arbitrary structure as well as zero/nonzero structure. In order to solve the problem of minimizing inputs for SSC, we transform it into a graph coloring problem according to a specific color change rule. Consequently, the corresponding MDP is defined. To further optimize the MDP, we incorporate a reinforcement learning method with directed graph neural networks. Secondly, we design an algorithm based on the aforementioned color change rule and degree distribution of the network. This approach facilitates the identification of the least number of control inputs, proving especially effective in small networks. Thirdly, we extend our method to various complex network models to explore the relationship between network topology and the minimum number of control inputs for SSC. Our research reveals that the minimum number of control inputs is determined by the average degree of the network and the input nodes tend to select nodes with low in-degree and avoid high-degree nodes.

\begin{figure*}[ht]
    \centering
    \resizebox*{16cm}{!}{\includegraphics{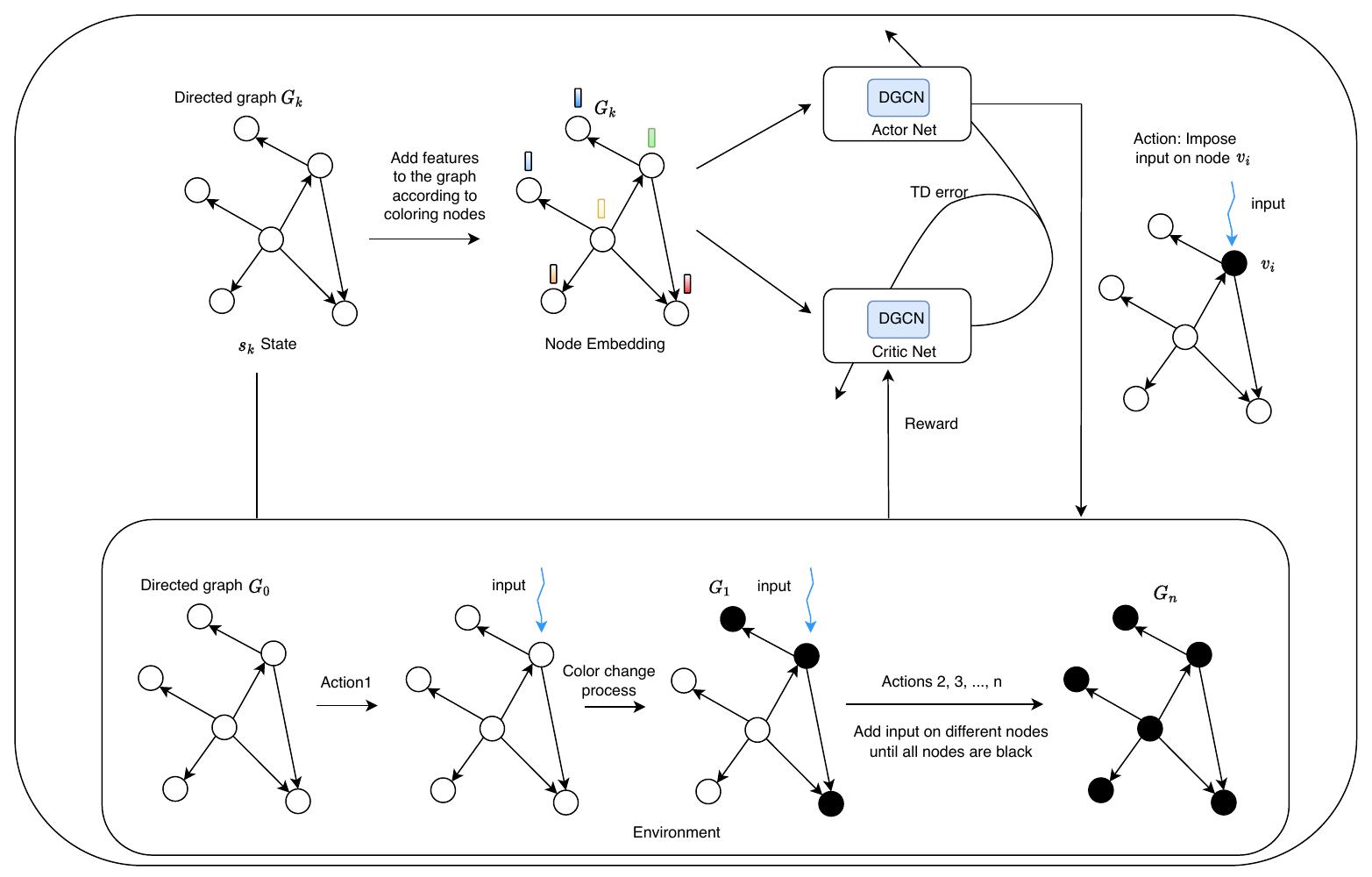}}
    \caption{This figure shows how to use reinforcement learning with directed graph neural networks to minimize the number of inputs to make the system controllable. The state represents the color of all nodes. The feature vector of each node is obtained by the color of nodes. Input the graph with feature vectors to the critic net to value the current state. Input the graph with feature vectors to the actor net to generate an action to color a node black. Coloring the graph according to the specific color change rule to get the next state. Repeat these steps until all nodes are colored black.  Calculate the reward and state values to update parameters of Actor Nets and Critic Net.}
    \label{fig:2}
\end{figure*}

\section{Methods}
\subsection{Graph basics}
We consider a complex network represented by a simple directed graph $G(\textbf{A}, \textbf{B})=(\textbf{V}, \textbf{E})$ where \textbf{A} is the connection matrix, \textbf{B} is the input matrix, $\textbf{V}$ is the node set and $\textbf{E}$ is the edge set respectively. $\textbf{V}=\{v_1, v_2, \cdots, v_N\}$ and $\textbf{E} \subseteq \textbf{V} \times \textbf{V}$. $\textbf{A} \in \{0, *, ?\}^{N \times N}$, where $0$ represents fixed zero, $*$ represents nonzero and $?$ represents any arbitrary value. $e_{j, i} \in \textbf{E}$ if and only if $A_{ij} = *$ or $A_{ij}=?$. If $e_{i, j} \in \textbf{E}$, node $v_j$ is the out-neighbour of node $v_i$. To distinguish between $*$ and $?$ entries in $\textbf{A}$, two subsets of $\textbf{E}$ are defined as $\textbf{E}_{*}$ and $\textbf{E}_{?}$. $e_{j, i} \in \textbf{E}_{*}$ if and only if $A_{ij}=*$. $e_{j, i} \in \textbf{E}_{?}$ if and only if $A_{ij}=?$. Edges in $\textbf{E}_{*}$ and $\textbf{E}_{?}$ are represented by solid and dash arrows respectively in visualization.

\subsection{System model}

Consider a complex system described by a directed weighted network of $N$ nodes, the dynamics of a linear time-invariant (LTI) system can be described as 
\begin{equation}
    \dot{\textbf{x}}(t)=\textbf{Ax}(t)+\textbf{Bu}(t),
\end{equation}
where $\textbf{x}(t)=(x_1(t),x_2(t), \cdots x_N(t))^{\top} \in \mathbb{R}^N$ captures the state of each node at time $t$. $\textbf{A}\in \mathbb{R}^{N \times N}$ is an $N \times N$ matrix describing the weighted connection of the network. The matrix element $a_{ij} \in \mathbb{R}$ gives the strength that node $j$ affects node $i$. $\textbf{B}\in \mathbb{R}^{N \times M}$ is an $N \times M$ input matrix ($M \le N$) identifying the nodes that are controlled by the time-dependent input vector $\textbf{u}(t) = (u_1(t), u_2(t), \cdots, u_M(t)) \in \mathbb{R}^{M}$ with $M$ independent signals imposed by the controller. The matrix element $b_{ij} \in \mathbb{R}$ represents the coupling strength between the input signal $u_j(t)$ and node $i$. In this paper, nodes directly controlled by input signals are called input nodes. Let $\textbf{V}_l = \{l_1, l_2, \cdots, l_M\} \subseteq \textbf{V}$ be the set of input nodes, then
\begin{equation}
    B_{ij}:=\left\{
    \begin{aligned}
        &1 \quad {\rm if} \quad v_i =l_j\\
        &0 \quad {\rm otherwise}. 
    \end{aligned}
    \right.
\end{equation}

Kalman's rank condition states that the LTI system is controllable if and only if the $N \times NM$ controllability matrix
\begin{equation}
    \textbf{C} \equiv [\textbf{B}, \textbf{AB}, \textbf{A}^2\textbf{B},\cdots,\textbf{A}^{N-1}\textbf{B}]
\end{equation}
has full rank, i.e.,
\begin{equation}
    {\rm{rank}} ~\textbf{C} = N.
\end{equation}

\begin{definition}
    A structural LTI system $\textbf{A}, \textbf{B}$ is called strong structural controllable (SSC) if it is controllable for all its numerical realizations $(\Tilde{\textbf{A}}, \Tilde{\textbf{B}})$.
\end{definition}

In this paper, the entries in $\textbf{A}$ are denoted by ${0, *, ?}$. The system $(\textbf{A}, \textbf{B})$ is SSC if the system satisfies ${\rm rank}~\textbf{C} =N$ for any admissible numerical realizations with zero/nonzero/arbitrary structure. To establish the necessary and sufficient conditions for SSC, we need to introduce the color change rule and the definition of derived set.

\begin{definition}\label{def: color}
    (Color change rule for zero/nonzero/arbitrary structure). Given a graph $G(\textbf{A}, \textbf{B})=(\textbf{V}, \textbf{E})$ where each node is initially colored either white or black, the color change rule is defined as follow: if node $v_i$ has exactly one white out-neighbour node $v_j$ and $e_{i,j} \in \textbf{E}_*$, then node $v_j$ is changed to black.
\end{definition}
If node $v_j$ changes color to black because of node $v_i$, then we say that $v_i$ forces $v_j$ to be black and is denoted  by $v_i \to v_j$. This color change rule is different from that in \cite{trefois2015zero, yaziciouglu2022strong}, where we consider $\textbf{E}_{?}$ exists in the graph. 

\begin{definition}
    (Derived Set). Given an initial set of black nodes $\textbf{V}_l \subseteq \textbf{V}$ (called the input set) in graph $G(\textbf{A}, \textbf{B})=(\textbf{V}, \textbf{E})$ and $G^{*}(\bar{\textbf{A}}, \textbf{B}) = (\textbf{V}, \textbf{E}')$, where $\bar{\textbf{A}}$ is the pattern matrix obtained from $\textbf{A}$ by modifying the diagonal entries of $\textbf{A}$ as follows:
\begin{equation}\label{equ: A}
    \bar{A}_{ii}:=\left\{
    \begin{aligned}
        * \quad & \quad {\rm if} \quad A_{ii} = 0 \\
        ? \quad & \quad {\rm otherwise,}
    \end{aligned}
    \right.
\end{equation} 
repeat the color change rule until no more white nodes can be colored to black. The set of all black nodes exists in $G(\textbf{A}, \textbf{B})$ is defined as derived set by ${\rm dset}(G, \textbf{V}_l) \subseteq \textbf{V}$. An input set $\textbf{V}_l$ is called a zero forcing set (ZFS) if ${\rm dset}(G, \textbf{V}_l) = {\rm dset}(G^{*}, \textbf{V}_l) = \textbf{V}$.  
\end{definition}

\begin{figure}[htbp]
    \centering
    \resizebox*{8cm}{!}{\includegraphics{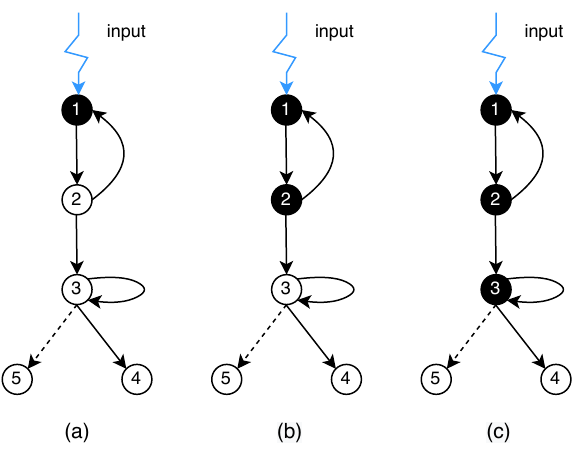}}
    \caption{This figure shows the color-change rules. (a) Node 1 is the input node which is colored black at first. (b) node 1 colors 2. (c) node 2 colors node 3.}
    \label{fig: color1}
\end{figure}

\begin{figure}[htbp]
    \centering
    \resizebox*{6cm}{!}{\includegraphics{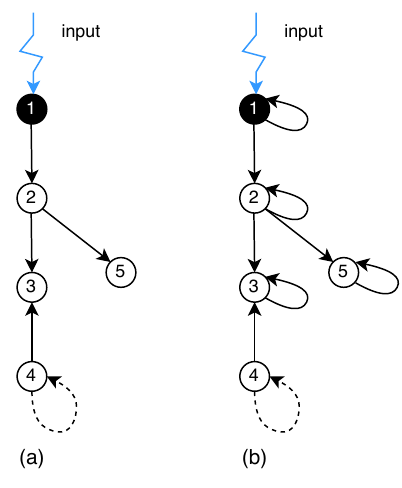}}
    \caption{(a) is the original graph $G(\textbf{A}, \textbf{B})$ and (b) is $G(\bar{\textbf{A}}, \textbf{B})$. The input node is $v_1$. According to the color change rule, the set of all black nodes in (a) and (b) is $(1, 2)$ and $(1, 2, 3, 5)$, respectively. The derived set ${\rm dset}(G, (1))$ is $(1, 2)$.}
    \label{fig: modify_graph}
\end{figure}

\begin{definition}\label{def: 4}
    (Minimum input nodes for SSC). The minimum input nodes for SSC is equal to the minimum size of zero forcing set of graph $G$ is defined as 
    \begin{equation}
        Z(G) = {\rm min}\{|\textbf{V}_l| \big| \textbf{V}_l \subseteq \textbf{V}, {\rm dset}(G, \textbf{V}_l)={\rm dset}(G^{*}, \textbf{V}_l) = \textbf{V} \}  
    \end{equation}    
\end{definition}

\begin{theorem}(proved in \cite{jia2020unifying})\label{theorem: 1}
    The system $\textbf{A}, \textbf{B}$ is controllable if and only if the following two conditions hold.
    
\noindent 1) Matrix $[\textbf{A}, \textbf{B}]$ has full row rank for all admissible numerical realizations.

\noindent 2) Matrix $[\bar{\textbf{A}}, \textbf{B}]$ has full row rank for all admissible numerical realizations, where $\bar{\textbf{A}}$ is obtained from $\textbf{A}$ according to equation~(\ref{equ: A}).
\end{theorem}

\begin{theorem}\label{theorem: 2}
 Let $\textbf{A} \in \{0, *, ?\}^{N \times N}$.
 $\textbf{V}_l \subseteq \textbf{V}$ is the set of input nodes. The system $(\textbf{A}, \textbf{B})$ is SSC if and only if

\noindent 1) ${\rm dset}(G(\textbf{A}, \textbf{B}), \textbf{V}_l) = \textbf{V}$;

\noindent 2) ${\rm dset}(G^{*}(\bar{\textbf{A}}, \textbf{B}), \textbf{V}_l) = \textbf{V}$ where $\bar{\textbf{A}}$ is obtained from $\textbf{A}$ by modifying the diagonal entries of $\textbf{A}$ as equation (\ref{equ: A}).
 
\end{theorem}

\begin{proof}
    First, we prove that if the matrix $[\textbf{A} \ \textbf{B}]$ is full row rank for all admissible numerical realization, $\textbf{V}_l$ is the derived set of the graph $G(\textbf{A}, \textbf{B})$. Here, we set $\textbf{M} = [\textbf{A} \ \textbf{B}]$. By means of a finite sequence of elementary row operations, any matrix can be transformed to a row echelon form as follow:
    \begin{equation}
        \begin{bmatrix}
            a_{i_1, j_1} & \otimes & \cdots & \otimes \\
            0 & a_{i_2, j_2} & \cdots &  \otimes \\
            \vdots & \vdots & \vdots &\vdots \\
            0 & 0 & \cdots & b_{i_n, j_n}  \\
        \end{bmatrix},
    \end{equation}
where $a_{i_1, j_1}, a_{i_2, j_2}, \cdots, b_{i_n, j_n}$ is nonzero element and $\otimes$ is zero, nonzero or arbitrary element. Since $\textbf{M}$ is full rank, rows with all zero elements do not exist. We use $a_{i_1, j_1}, a_{i_2, j_2}, \cdots, b_{i_n, j_n}$ to represent the first left nonzero element in each row. The dimension of $i_1, i_2, \cdots, i_n$ is $N$. In $j_1$ column, only one nonzero element $a_{i_1,j_1}$ exists, which means that the node $v_{i_1}$ is the only white out-neighbour of node $v_{j_1}$. Therefore, node $v_{i_1}$ is colored black. Then, in the $j_2$ column, $v_{i_2}$ is the only white out-neighbour node of node $v_{j_2}$ because another out-neigbour of node $v_{j_2}$ is node $v_{i_1}$ which has been colored black. By repeating the above steps, all nodes will be colored black at last. $G(\textbf{A}, \textbf{B})$. This proves that if the matrix $[\textbf{A} \ \textbf{B}]$ is full row rank for all admissible numerical realization, ${\rm dset}(G(\textbf{A}, \textbf{B}), \textbf{V}_l) = \textbf{V}$. 

Now, we prove that only if the matrix $[\textbf{A} \ \textbf{B}]$ is full row rank for all admissible numerical realization, $\textbf{V}_l$ is the derived set of the graph $G(\textbf{A}, \textbf{B})$. Without loss of generality the matrix $\textbf{A}$ can be partitioned as 

\begin{equation}
    \textbf{A} = 
    \begin{bmatrix}
        A_{11} & A_{12} & A_{13} & A_{14} \\
        A_{21} & A_{22} & A_{23} & A_{24} \\
        A_{31} & A_{32} & A_{33} & A_{34} \\
        A_{41} & A_{42} & A_{43} & A_{44} 
    \end{bmatrix}
\end{equation}
where the diagonal blocks/elements $\textbf{A}_{11}, \textbf{A}_{22}, \textbf{A}_{33}, \textbf{A}_{44}$ represent the nodes in $\textbf{V}_L \setminus \{v_i\}$, the node $v_i$. the node $v_j$ and the remaining nodes, respectively. Suppose that $v_j$ is a node will be enforced to be black. Let $\xi = \{\xi_1, \xi_2, \xi_3, \xi_4 \}$ be a vector of $[\textbf{A} \ \textbf{B}]^{\bot}$, the orthogonal subspace of $[\textbf{A} \ \textbf{B}]$. Then we have $\xi[\textbf{A} \ \textbf{B}] = 0$, which equals to

\begin{equation}
[\xi_1 \ \xi_2 \ \xi_3 \ \xi_4]
    \begin{bmatrix}
       A_{11} & A_{12} & A_{13} & A_{14} &  \textbf{I}_{M-1} & 0 \\
       A_{21} & A_{22} & A_{23} & A_{24} & 0 & 1\\
       A_{31} & A_{32} & A_{33} & A_{34} & 0 & 0\\
       A_{41} & A_{42} & A_{43} & A_{44} & 0 & 0    
    \end{bmatrix}=0.
\end{equation}

Then we have $\xi_1=\xi_2=0$ and 
\begin{equation}
    [\xi_3 \ \xi_4] 
    \begin{bmatrix}
     A_{31} & A_{32} & A_{33} & A_{34} \\
     A_{41} & A_{42} & A_{43} & A_{44}    
    \end{bmatrix}=0
\end{equation}
There exist three cases that node $v_j$ can be enforced black. First, node $v_j$ is colored by node $v_i$ in $\textbf{V}_L$. Then $A_{32} \neq 0$ and $\xi_3$ must be $0$. In the second case, node $v_j$ is colored by itself, which indicates that $A_{33} \ne 0$ and $\xi_3$ must be $0$. In the third case, node $v_j$ is colored black by the remaining nodes in $A_{43}$. Then one element in $A_{34}$ is nonzero, which means that $\xi_3 = 0$. Therefore, in all these three cases, $\xi_3=0$. Then we have 
\begin{equation} \label{equ: 11}
    [\xi_4]
    \begin{bmatrix}
     A_{41} & A_{42} & A_{43} & A_{44}   
    \end{bmatrix}=0.
\end{equation}

Now, we consider add the node $v_j$ to $\textbf{V}_L$ to get $\textbf{V}'_L = \textbf{V}_L \cup {v_j}$. Let $\eta = \{\eta_1, \eta_2, \eta_3, \eta_4 \}$ be a vector of $[\textbf{A} \ \textbf{B}']^{\bot}$, where $\textbf{B}'$ is corresponding to $\textbf{V}'_L$. Then we have
\begin{equation} 
[\eta_1 \ \eta_2 \ \eta_3 \ \eta_4]
    \begin{bmatrix}
       A_{11} & A_{12} & A_{13} & A_{14} &  \textbf{I}_{M-1} & 0 & 0 \\
       A_{21} & A_{22} & A_{23} & A_{24} & 0 & 1 & 0\\
       A_{31} & A_{32} & A_{33} & A_{34} & 0 & 0 & 1\\
       A_{41} & A_{42} & A_{43} & A_{44} & 0 & 0 & 0  
    \end{bmatrix}=0,   
\end{equation}
which requires $\eta_1=\eta_2=\eta_3=0$ and 
\begin{equation} \label{equ: 13}
    [\eta_4] 
    \begin{bmatrix}
    A_{41} & A_{42} & A_{43} & A_{44}  
    \end{bmatrix} = 0.
\end{equation}
From equation \eqref{equ: 11} and \eqref{equ: 13}, we know that $[\textbf{A} \ \textbf{B}]$ has the same dimension of $[\textbf{A} \ \textbf{B}']$. If ${\rm dset}(G(\textbf{A}, \textbf{B}), \textbf{V}_l) = \textbf{V}$, by repeating the process of adding black nodes into $\textbf{V}_L$, $\textbf{V}'_L = \textbf{I}_N$ and $\eta_4 = 0$. $[\textbf{A} \ \textbf{B}']$ has full rank, which means that $[\textbf{A} \ \textbf{B}]$ has full rank. Hence, the only if part has been proved.

Therefore, we have proved that $[\textbf{A} \ \textbf{B}]$ has full rank for all admissible numerical realization if and only if $\textbf{V}_l$ is the zero forcing set of $G(\textbf{A}, \textbf{B})$. Combining this conclusion with Theorem~\ref{theorem: 1}, we can prove Theorem~\ref{theorem: 2}.
\end{proof}

\begin{theorem}\label{theorem: 3}(proved in \cite{trefois2015zero} (Theorem 5.5)) Let $G$ be a loop directed graph on $n$ vertices with pattern $\textbf{A}$ and $\textbf{V}_l$ be an input set with cardinality $m \le n$. System $(\textbf{A}, \textbf{B})$ is SSC if and only if 

\noindent 1) $\textbf{V}_l$ is a zero forcing set of $G$

\noindent 2) $\textbf{V}_l$ is a zero forcing set of $G^*$ for which there is a chronological list of forces that does not contain any force of the form $i \to i$ with $i \in \textbf{V}_{\rm loop}$, where $\textbf{V}_{\rm loop}$ is a set of nodes in the original graph with self-loop. $G^*$ denotes the graph obtained from graph $G$ by putting a loop on each node of $G$.    
\end{theorem}

Actually, Theorem~\ref{theorem: 2} is equivalent to Theorem~\ref{theorem: 3} if we only consider the zero/nonzero structure. Condition 1 in Theorem~\ref{theorem: 2} and Theorem~\ref{theorem: 3} are equivalent for zero/nonzero structure. For zero/nonzero structure, in condition 2 of Theorem~\ref{theorem: 2}, there exist 2 types of self-loop in the original graph $G$: $A_{ii}=0, *$ and three types of self-loop in the modified graph $G^*$: $A_{ii}=0, *, ?$.  If $A_{ii}=0$, $\bar{A}_{ii}=*$. This is the same with condition 2 in Theorem~\ref{theorem: 2} that add self-loop to nodes without self-loop in the original graph $G$. If $A_{ii}=*$, $\bar{A}_{ii}=?$. In this case, nodes with solid self-loop will have dashed self-loop which cannot force themselves black. This is equivalent to the condition 2 in Theorem~\ref{theorem: 3} that the chronological list of forces does not contain any force of the form $i \to i$ with $i \in \textbf{V}_{\rm loop}$. Therefore, Theorem~\ref{theorem: 2} is also valid for zero/nonzero structure. The color change rule is valid in both zero/nonzero/arbitrary and zero/nonzero structure.

\begin{proposition}\label{pro: 1}
    For any graph $G=(\textbf{V}, \textbf{E})$ with input nodes $\textbf{V}_l \subseteq \textbf{V}$, $|{\rm dset}(G, \textbf{V}_l)| = |{\rm dset}(G, {\rm dset}(G, \textbf{V}_l))|$, where $|{\rm dset}(G, \textbf{V}_l)|$ is the size of the derived set corresponding to the input set $\textbf{V}_l$.
\end{proposition}

\begin{proof}
In the case ${\rm dset}(G, \textbf{V}_l) = \textbf{V}$, $|{\rm dset}(G, \textbf{V}_l)| = |{\rm dset}(G, {\rm dset}(G, \textbf{V}_l))|=N$. The proposition is obviously true in this case. 
In the case ${\rm dset}(G, \textbf{V}_l) \subsetneqq \textbf{V}$, it is always true that $|{\rm dset}(G, \textbf{V}_l)| \leqslant |{\rm dset}(G, {\rm dset}(G, \textbf{V}_l))|$. If the proposition is not true, then we assume that there exists a node $v_j$ satisfying $v_j \in {\rm dset}(G, {\rm dset}(G, \textbf{V}_l))$ but $v_j \notin {\rm dset}(G, \textbf{V}_l)$. ${\rm dset}(G, \textbf{V}_l))= {\rm dset}(G, {\rm dset}(G, \textbf{V}_l))/v_j$. Since no nodes can be forced to black except the derived set ${\rm dset}(G, \textbf{V}_l))$, it is obvious no other nodes can be forced to black except ${\rm dset}(G, {\rm dset}(G, \textbf{V}_l))/v_j$. Therefore, $v_j$ does not exist. This proves that $|{\rm dset}(G, \textbf{V}_l)| = |{\rm dset}(G, {\rm dset}(G, \textbf{V}_l))|$.
\end{proof}

Proposition~\ref{pro: 1} states that imposing control signals on nodes in the derived set cannot enforce nodes not belonging to the derived set to be black.  
According to the color change rule in Definition~\ref{def: color}, it is obvious that the color change process is related to the in-degree and out-degree of a node. For example, if node $v_i$'s in-degree is zero, it cannot be colored black by any other nodes in the graph. We need to add input to it to make it an input node. Therefore, it is reasonable to make all nodes with zero in-degree input node nodes. Here, we design a simple algorithm based on Proposition~\ref{pro: 1} and degree (shown in Fig.~\ref{fig: alg1}). 

\begin{figure}[htbp]
    \centering
    \resizebox*{8cm}{!}{\includegraphics{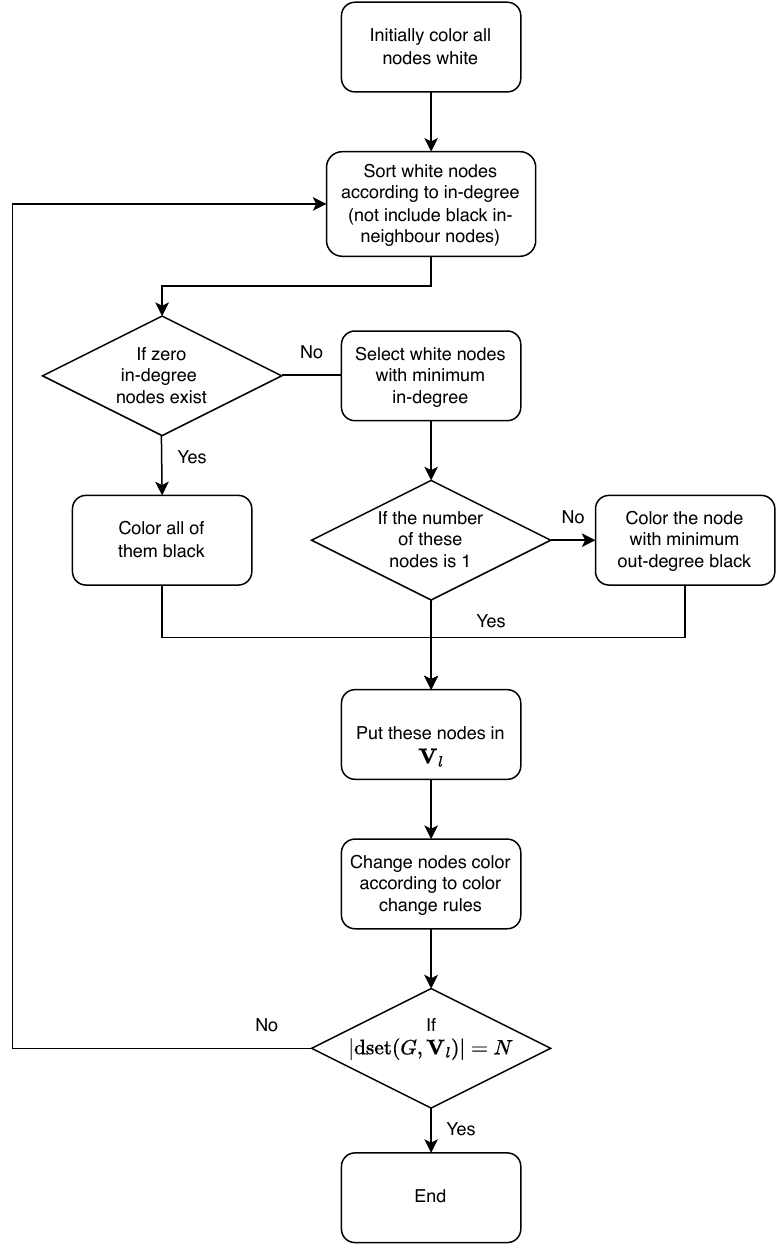}}
    \caption{This figure shows the algorithm based on Proposition~\ref{pro: 1} and degree.}
    \label{fig: alg1}
\end{figure}

\subsection{Color change process as a Markov decision process}
According to Definition~\ref{def: 4} and Theorem~\ref{theorem: 2}, finding the minimum number of input nodes for SSC is equivalent to finding the minimum number of input node nodes that can color all nodes in the graph black. We can formulate the color change process as an MDP to find the minimum number of input nodes for SSC. Since the color change rule is valid both in zero/nonzero and zero/nonzero/arbitrary structure, the MDP formulated according to the color change rule can also be applied in these two types of structures.

An MDP represents a possible way to formalize a decision-making process. Within this process, the decision maker is referred as the agent, and the surrounding in which it operates is called environment. When the agent is in a state $s\in S$, it selects a valid action $a \in A(s)$. Subsequently, the agent receives a reward $r$ based on the reward function $r(s, a)$, and transitions to a new state $s'$ according to the transition model $\mathcal{P}$ and its associated transition probability $P(s', a, s)$. The objective of the agent is to maximize the expected sum of rewards. This MDP is defined by the tuple $(S, A, \mathcal{P}, r, \gamma)$, where $\gamma$ represents the discount factor. The policy of agent is defined as $\pi_{\theta}$. The state value function of a given state $s$, based on the policy $\pi_{\theta}$, is denoted as $V^{\pi_{\theta}}(s)$.

Given an initial graph $G_0(\textbf{A}, \textbf{B}) = (\textbf{V}, \textbf{E})$, the initial set of input nodes $\textbf{V}_l = \varnothing$. The aim is to add the minimum number of nodes to $\textbf{V}_l$ to make all nodes black according to the color change rule in Definition~\ref{def: color}. Here, we cast the color change process as a sequential decision-making process and define the MDP as follows:

\noindent \textbf{State.} The fully observable state is the color of all nodes in $G_k$ at step $k$ denoted by $s_k$ depending on the derived set. We use $1$ to represent a black node and $0$ to represent a white node. For example, in Fig.~\ref{fig: modify_graph}, the derived set ${\rm dset}(G, \textbf{V}_l)=\{v_1, v_2\}$, $\textbf{V}_l={v_1}$. The state $s$ is denoted as $[1, 1, 0, 0, 0]$.

\noindent \textbf{Action.} The action is to add one input node to the graph at each step. The action space is the number of nodes in the graph. For a graph with $N$ nodes, the action space is $N$. To reduce the action space, the same action is not allowed. For example, node $v_1$ is selected as the input node node in the first action. Then node $v_1$ cannot be selected as the input node node again in the following steps. 

\noindent \textbf{State Transition.} After taking an action, $s_k$ transforms to $s_{k+1}$ according to the color change rule on the graph (shown in Fig~\ref{fig: transition}).

\noindent \textbf{Reward.} The reward $r_k$ depends on the size of derived set $|{\rm dset}(G, \textbf{V}_l)|$

\begin{equation}
 r_k = \left\{
\begin{aligned}
 & 100,\quad |{\rm dset}(G, \textbf{V}_l) |=|{\rm dset}(G^*, \textbf{V}_l) |=N, \\ 
 & -1, \quad {\rm Otherwise}.
\end{aligned}
\right.
\end{equation}

In MDP, return $R_k$ is defined as the sum of rewards from step $k$ to the end of the process. Return $R_k$ is calculated by
\begin{equation}
    R_k = r_k + \gamma r_{k+1} + \gamma^2 r_{k+2} + ... = \sum_{j=0}^{\infty}\gamma^j r_{k+j},
\end{equation}
where $\gamma$ is the decay factor.

\begin{figure*}[ht]
    \centering
    \resizebox*{16cm}{!}{\includegraphics{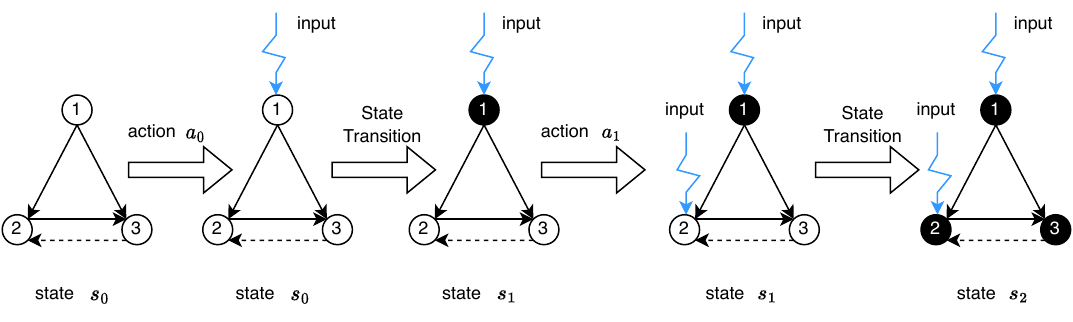}}
    \caption{This figure shows the transition of states according to the color change rule. At state $s_0$, according to the color change rule, in graph $G_0$, node $2$ and $3$ are black, but all nodes are white in $G_0^*$. Therefore, all nodes are white in state $s_0$. Then we take action $a_0$ to impose input on node $1$. All nodes in graph $G_1$ are black, but only node $1$ is black in $G_1^*$. So only node $1$ is black in state $s_1$. State $s_0$ transits to $s_1$ after taking action $a_0$. Then we take action $a_1$ to impose input on node $2$. State $s_1$ transits to state $s_2$ in which all nodes are black. In state $s_2$, the system is SSC and the MDP ends.} 
    \label{fig: transition}
\end{figure*}

\begin{figure*}[htbp]
    \centering
    \resizebox*{16cm}{!}{\includegraphics{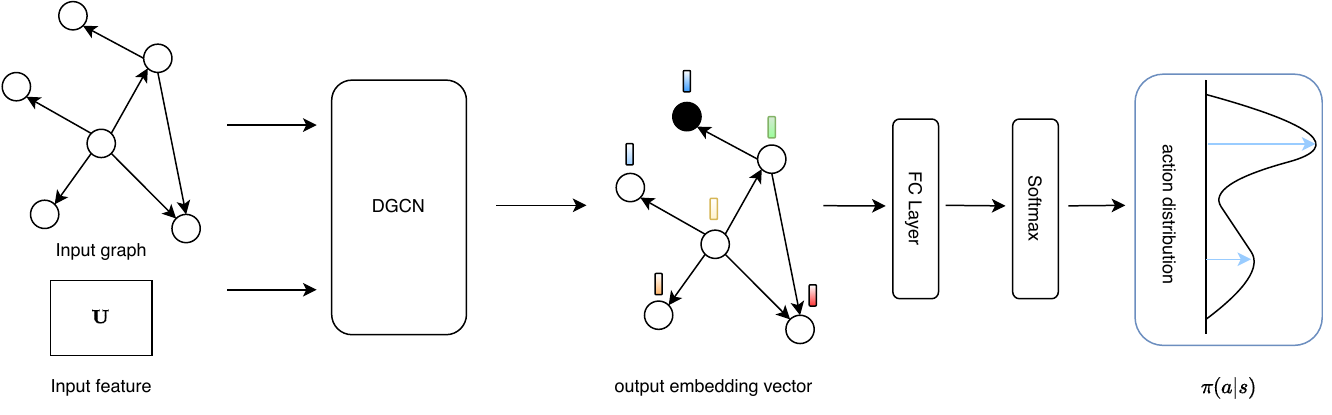}}
    \caption{This figure shows how the actor net generates actions. We input graph and feature vector of each node to DGCN to output embedding vector of each node. The embedding vectors are put into full connect layers and softmax function to get the action distribution.}
    \label{fig: DGCN}
\end{figure*}

\subsection{Represent graph information by directed graph neural network}
While the minimum input problem formulation as a MDP may allow us to solve it in reinforcement learning method, the number of states and the size of the action space  become intractable in a large graph. Besides,
in each step, the agent will take an action which adds a input node node to the graph, and then some white nodes will be forced to black to form a new state according to the color change rule. The action and the state are closely related to the structure of the graph. It is difficult for neural networks used in reinforcement learning methods like deep Q learning, actor critic network, to deal with graph structure data. Thus, we need a method which can represent the graph structure and generalize across similar states as well as actions. Graph neural network can be considered to solve this problem. Since the graph is directed, the directed graph neural network is employed here. Given an input graph $G=(\textbf{V}, \textbf{E})$ where nodes $v_i \in \textbf{V}$ have feature vectors $\textbf{u}_i$, its objective is to produce for each node $v_i$ an embedding vector $\textbf{z}_i$ that captures the structure of the directed graph and interactions between neighbours. Here, we employ the spectral-based GCN model $f(\textbf{U}, \textbf{A})$ for directed graphs that leverage the First- and Second- Order Proximity, called DGCN \cite{tong2020directed}. The multi-layer Graph Convolutional Network for directed graphs has the following layer-wise propagation rule:
\begin{equation}
    \textbf{H}^{l+1}= \Gamma(\textbf{H}^{l}, \textbf{A}),
\end{equation}
where $\Gamma$ is a fusion function such as normalization functions, summation functions and concatenation. $\textbf{H}^{l}$ is the matrix of activation in the $l^{th}$ layer and $\textbf{H}^{(0)} = \textbf{U}$, where $\textbf{U}=(\textbf{u}_1, \textbf{u}_2, \cdots, \textbf{u}_N)$. 

The action is predicted by a DGCN. Input the graph $G(\textbf{V}, \textbf{E})$ and feature vectors $\textbf{U}$ of nodes into the DGCN to produce embedding vectors for all nodes. Then input embedding vectors to FC Layer and Softmax function to output the probability of action distribution (shown in Fig.~(\ref{fig: DGCN}).

The value of the state $V^{\pi_{\theta}}(s_k)$ is predicted by a DGCN. Policy gradient expression is defined as:
\begin{equation}
    \nabla J(\theta)=\mathbb{E}[\sum_{k=0}^K\psi_k \nabla_{\theta} \log \pi_{\theta}(a_k|s_k)].
\end{equation}
$\psi_k$ has a lot of forms to choose, and temporal difference (TD) error is used in this paper as follow:
\begin{equation}
  \psi_k = r_k + \gamma V^{\pi_{\theta}}(s_{k+1}) - V^{\pi_{\theta}}(s_{k}).   
\end{equation}
The TD error is used to update parameters of actor net and critic net. Parameters $\theta$ in actor network are updated by $\theta=\theta+\alpha_{\theta} \nabla J(\theta)$ and parameters $\omega$ in critic network are updated by $\omega = \omega+\alpha_{\omega} \nabla J(\omega)$.

\begin{algorithm}[ht]\label{algorithm 2}
\caption{Actor-Critic method with Directed Graph Convolutional Network}
\begin{algorithmic}[1]
    \STATE Initialize parameters $\theta$ in actor net, parameters $\omega$ in critic net, learning rate $\alpha_{\theta},  \alpha_{\omega}$;
    \FOR{episode $i=1, I$}
    \STATE Input the initial graph $G_0$ to get the initial state $s_0$ according to the color change rules in \ref{def: color};
    \STATE Obtain the feature vector $\textbf{u}_i$ according to the state $s_0$ for each node;
    \WHILE{step $k=0, K$}
    \STATE Input feature vectors and the graph to critic net to estimate the value of current state $r_k$;
    \STATE Input feature vectors and the graph to actor net to generate an action $a_k$ by policy $\pi_{\theta}$;
    \STATE State transits from $s_k$ to the next state $s_{k+1}$ after taking action $a_k$;
    \STATE Input feature vectors $\textbf{u}'_i$ and state $s_{k+1}$ to critic net to estimate the value of current state $r_{k+1}$;
    \ENDWHILE
    \STATE Sample trajectory $\{s_0, a_0, r_0, s_1, a_1, r_1,\ldots, s_K, r_K\}$
    \STATE Compute TD error $\delta_k = r_k + \gamma V^{\pi_{\theta}}(s_{k+1}) - V^{\pi_{\theta}}(s_{k})$
    \STATE Update parameters of critic net by $\omega = \omega+\alpha_{\omega} \nabla J(\omega)$
    \STATE Update parameters of actor net  by
    $\theta=\theta+\alpha_{\theta} \nabla J(\theta)$ 
    \ENDFOR
\end{algorithmic}
\end{algorithm}

\section{Results}
\subsection{A simple graph}
Here, we apply our method to a simple case to show how it works. The graph $G$ and the modified graph $G^*$ are shown in Fig.~\ref{fig: case1}.  In initial graph $G$, the derived set is $\{ v_0, v_2, v_3, v_4, v_5, v_7\}$. In the initial graph $G^*$, the derived set is $\varnothing$. ${\rm dset
}(G, \varnothing) \cap {\rm dset
}(G^*, \varnothing)=\varnothing$. No node can be colored black without input in the initial graph. The initial state $s_0=[0, 0, 0, 0, 0, 0, 0, 0, 0, 0]$. According to the reinforcement learning method in this paper, the first action is to impose control signals on node $v_1$ and $\textbf{V}_l = \{v_1\}$. ${\rm dset
}(G, \textbf{V}_l) = \{v_0, v_1, v_2, v_4\}$. ${\rm dset
}(G^*, \textbf{V}_l) = \{v_0, v_1, v_2, v_3, v_4, v_5, v_6, v_7, v_8\}$. We can get ${\rm dset
}(G, \textbf{V}_l) \cap {\rm dset
}(G^*, \textbf{V}_l)=\{v_0, v_1, v_2, v_4\}$. State $s_1=[1, 1, 0, 1, 0, 0, 0, 0, 0, 0]$. The second action is $v_6$. ${\rm dset
}(G, \textbf{V}_l) \cap {\rm dset
}(G^*, \textbf{V}_l)=\{v_0, v_1, v_2, v_4, v_6\}$. State $s_2=[1, 1, 0, 1, 0, 1, 0, 0, 0, 0]$. The third action is $v_9$. ${\rm dset
}(G, \textbf{V}_l) \cap {\rm dset
}(G^*, \textbf{V}_l)=\{v_0, v_1, v_2, v_3, v_4, v_5, v_6, v_7, v_8, v_9\}$. The state $s_3=[1, 1, 1, 1, 1, 1, 1, 1, 1, 1]$. All nodes are colored black after these three actions. $\textbf{V}_l=\{v_1, v_6, v_9\}$ is a zero forcing set and $Z(G)=3$. Then we apply the algorithm based on degree designed in this paper. According to the algorithm based on degree, the input set is $\textbf{V}_l=\{v_1, v_6, v_8, v_9\}$. These two methods are valid to find the zero forcing set and the reinforcement learning method has a better result.

\begin{figure}[ht]
    \centering
    \resizebox{8cm}{!}{\includegraphics{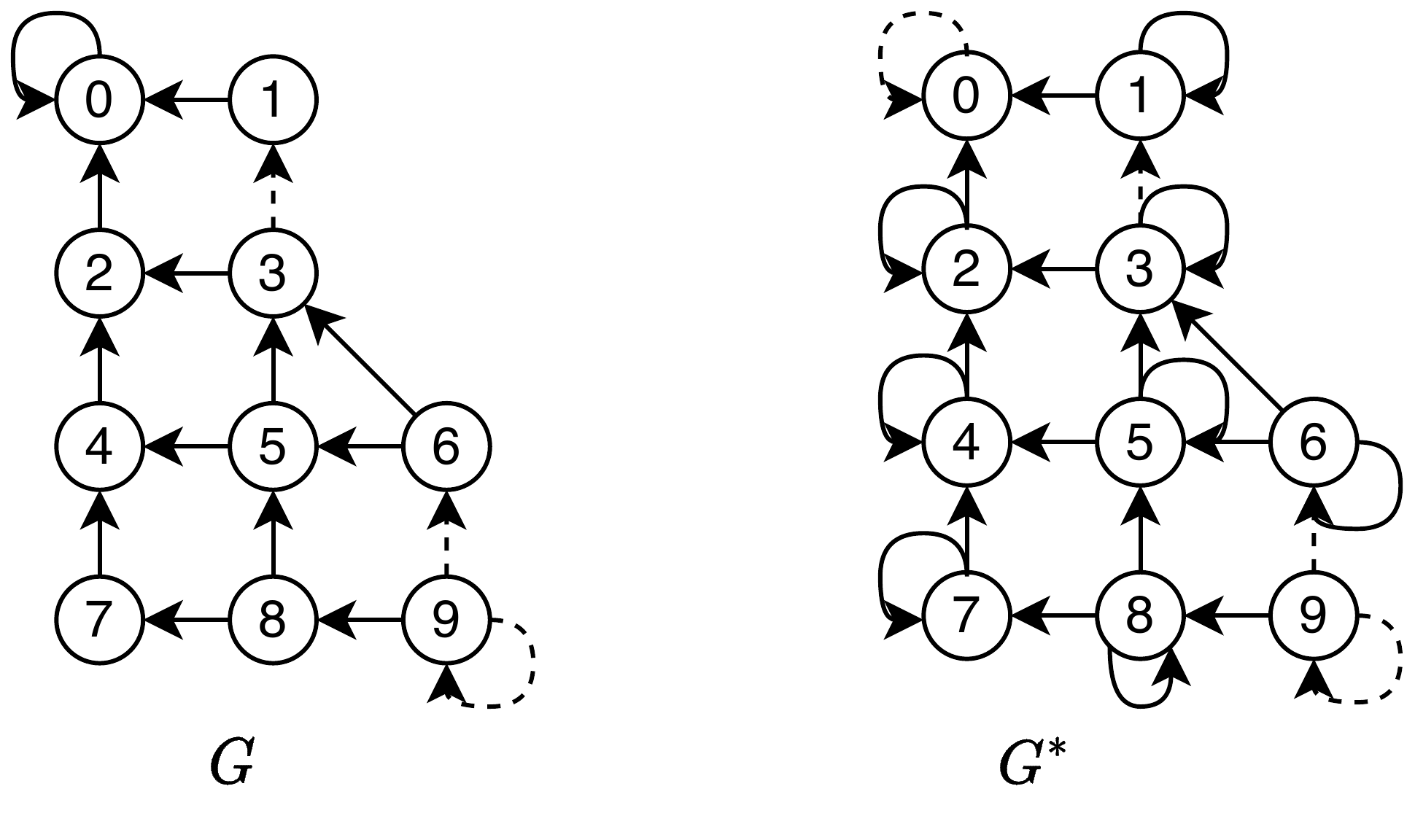}}
    \caption{A simple case, graph $G$ and 
 its modified graph $G^*$.}
    \label{fig: case1}
\end{figure}

\subsection{Influence social networks}
In this case, we employ our method to an influence social network. In social networks, according to French's formal theory \cite{french1956formal}, the influence effect in social networks is determined to be proportional to the size of the difference between opinions $g(x_i(t), x_j(t)) = a_{ji}(x_j(t)-x_i(t))$, where $a_{ij}$ is the strength of the effect. Here, we also assume that the self-dynamics function $f(\cdot)$ is linear as $f(x_i(t)) = a_{ii}x_i(t)$. Then, influence on node $i$ can be described as:
\begin{equation}
    \frac{dx_i(t)}{dt}=a_{ii}x_i(t) + \sum_{j=1, j \ne i}^N a_{ij}(x_j(t)-x_i(t)).
\end{equation}
Then, the system can be written as a matrix form 
\begin{equation}
    \dot{\textbf{x}}(t) = \textbf{W}^{\top}\textbf{x}(t),
\end{equation}
where $\textbf{W} \in \mathbb{R}^{N \times N}$ is a matrix of $W_{ij}$. $W_{ii}=a_{ii}+\sum_{j=1, j \ne i}^Na_{ij}, W_{ij}=a_{ij}$. Considering the control imposed on the system, the dynamics can be described as 
\begin{equation}\label{equ:weight}
    \dot{\textbf{x}}(t) = \textbf{W}^{\top}\textbf{x}(t) + \textbf{B}\textbf{u}(t).
\end{equation}
Obviously, the parameter of $x_i(t)$ is $a_{ii}+\sum_{j=1, j \ne i}^Na_{ij}$ which partly depends on $a_{ij}$. Parameters of $x_i(t)$, $a_{ii}+\sum_{j=1, j \ne i}^Na_{ij}$, can be zero, nonzero or arbitrary. Therefore, these social network is zero/nonzero/arbitrary structure. In a directed graph, we use $k_i^{\rm in}, k_i^{\rm out}, k_i$ to represent the in-degree, out-degree and degree of node $v_i$ respectively. $k_i = k_i^{\rm in}+k_i^{\rm out}$. $k^{\rm in}_i$ is the number of in edge $e_{j,i} \in \{\textbf{E}_*, \textbf{E}_?\}$ of node $v_i$. $k^{\rm out}_i$ is the number of in edge $e_{i,j} \in \{\textbf{E}_*, \textbf{E}_?\}$ of node $v_i$. Here, we use the Twitter dataset with Feminism topics. This social network consists of 180 nodes and about 1000 edges. 

By using our method, the minimum number of zero forcing set $Z(G)$ is $83$. The in-degree distribution, out-degree distribution, and degree distribution of nodes in the graph and the zero forcing set are shown in Fig~\ref{fig: degree distribution}. The average degree, in-degree, and out-degree of nodes in the graph and zero forcing set are shown in Table~\ref{tab: degree}. It is obvious that zero forcing set tends to select low in-degree nodes to be input nodes and avoid nodes with high in-degree, out-degree, and degree. Compared with out-degree, in-degree is a more important factor for nodes to be selected as potential input nodes. Note that there are no nodes with $0$ in-degree. The lowest in-degree is $1$. If nodes with $0$ in-degree exist in the graph, all of them need to be imposed control signals because other nodes cannot enforce nodes with $0$ in-degree to be black. In the first subfigure of Fig~\ref{fig: degree distribution}., nodes with lowest in-degree are more than $90 \%$ in zero forcing set. This result identifies that the nodes with the lowest in-degree are critical to SSC of the influence social network. This is quite different from the previous understanding of influence in social networks where nodes with high degree as hubs play important roles. This is because the high-degree nodes usually have high in-degree and out-degree. Imposing control signals on a node with high degree may only control this node itself. Therefore, selecting nodes with high degree is not efficient for SSC. 

\begin{figure*}[ht]
    \centering
    \resizebox{16cm}{!}{\includegraphics{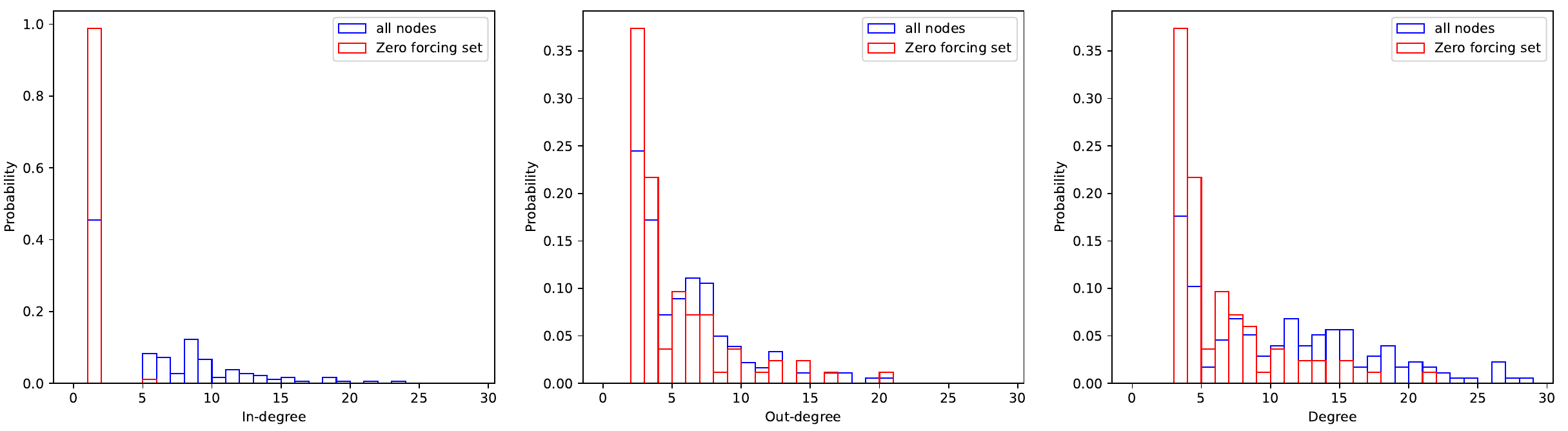}}
    \caption{This figure shows the in-degree distribution, out-degree distribution, and degree distribution of the graph and zero forcing set.}
    \label{fig: degree distribution}
\end{figure*}

\begin{figure*}[ht]
    \centering
    \resizebox{16cm}{!}{\includegraphics{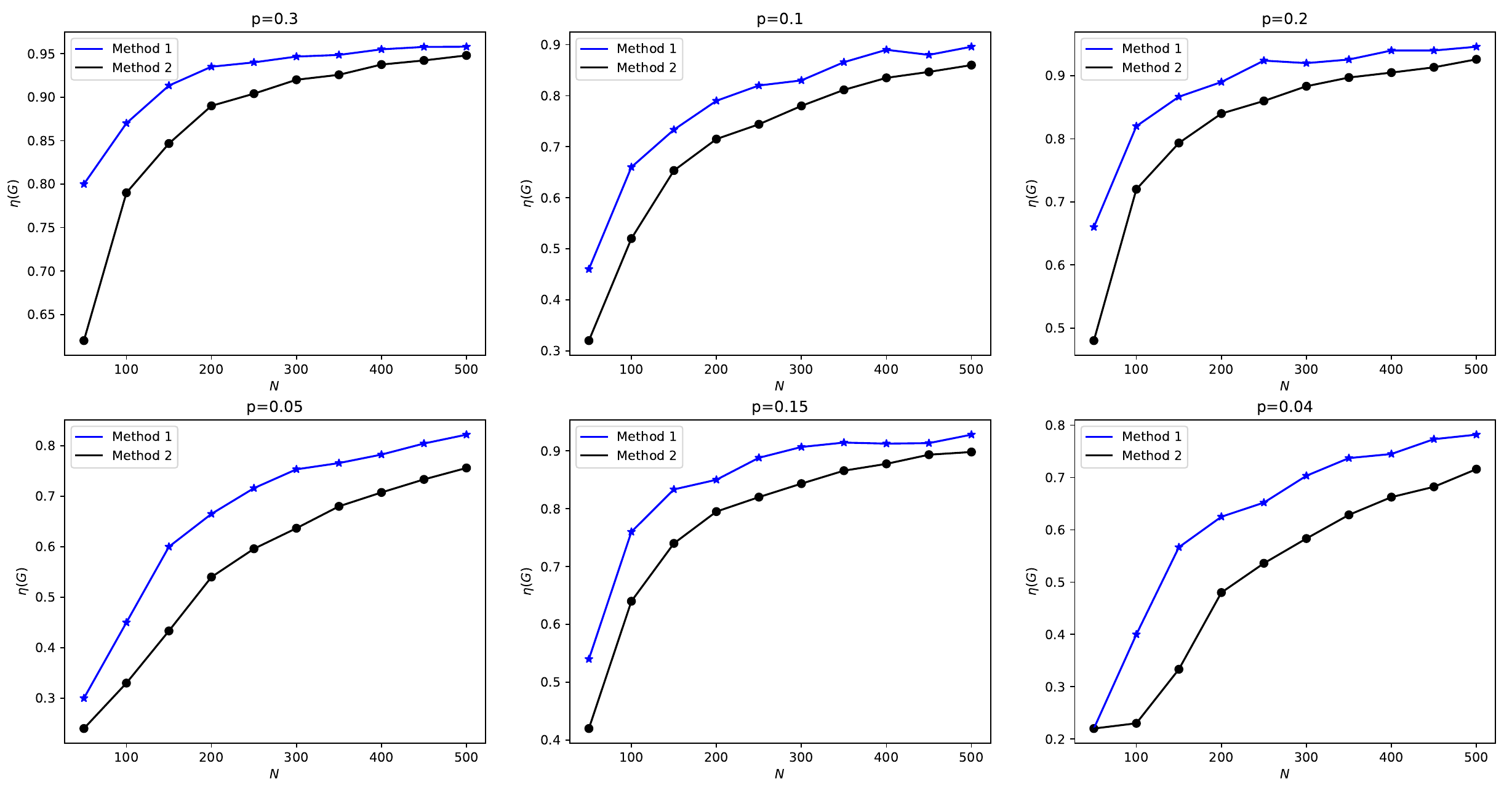}}
    \caption{Comparing reinforcement learning method (Method 2) in this paper and the algorithm based on degree (Method 1) in different Erdos Renyi random graphs. $N$ is the number of nodes in a graph. $\eta(G)$ is the ratio of input nodes for SSC and total nodes in the graph.}
    \label{fig: compare}
\end{figure*}

\begin{figure}[htbp]
    \centering
    \resizebox{8cm}{!}{\includegraphics{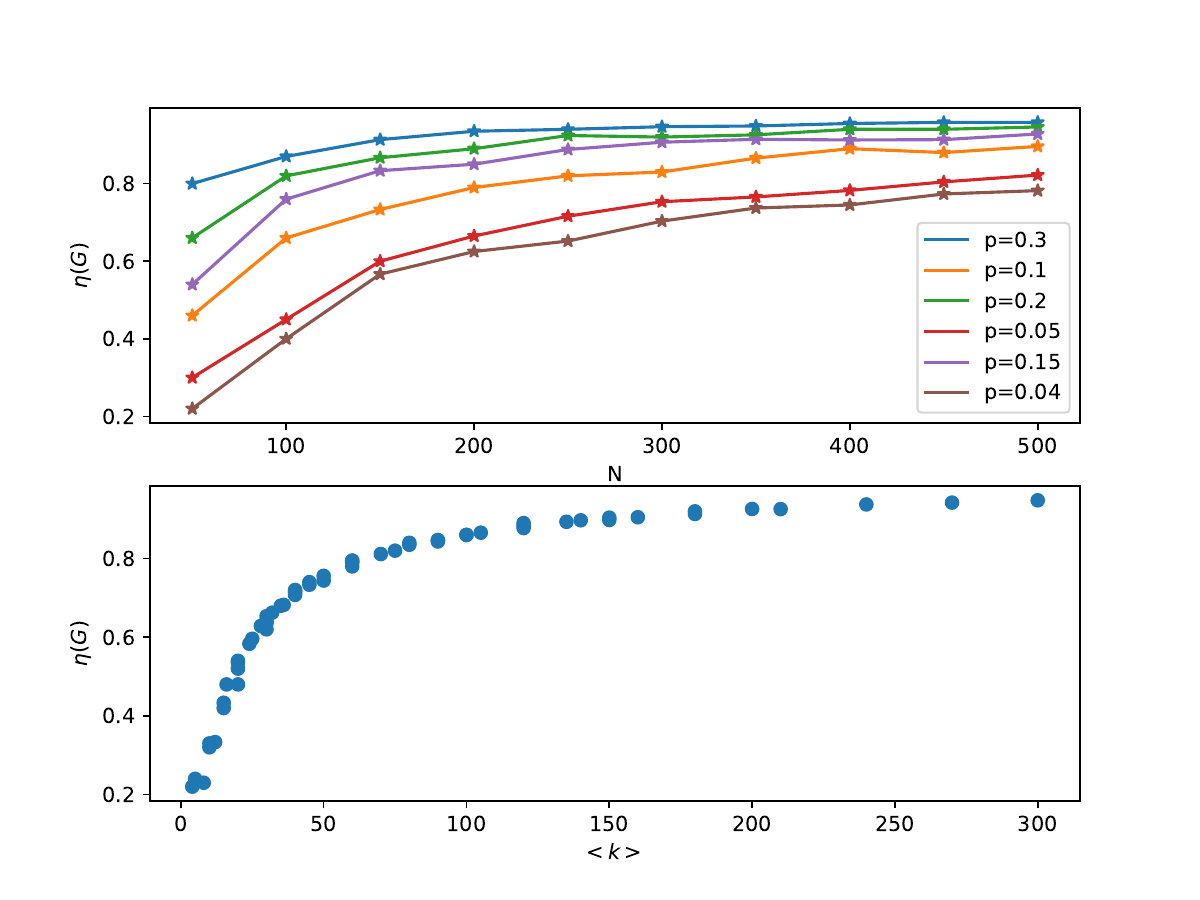}}
    \caption{The first subfigure shows $\eta(G)$ in different ER random graphs. The second subgraph shows the relationship between $\eta(G)$ and the average degree $<k>$ of graphs.}
    \label{fig: random}
\end{figure}

\begin{table}[]\label{tab: degree}
\caption{Properties of zero forcing set and the whole graph}
\begin{tabular}{|l|l|l|l|l|}
\hline & \begin{tabular}[c]{@{}l@{}}Number\\ of nodes\end{tabular} & \begin{tabular}[c]{@{}l@{}}Average\\ degree\end{tabular} & \begin{tabular}[c]{@{}l@{}}Average\\ in-degree\end{tabular} & \begin{tabular}[c]{@{}l@{}}Average\\ out-degree\end{tabular} \\ \hline
All nodes     & 180     & 10.88        & 5.44      & 5.44     \\ \hline
Zero forcing set & 83   & 5.63      & 1.05     & 4.58    \\ \hline
\end{tabular}
\end{table}

\subsection{Different random network models}
Here, we apply our method in different random network models to optimize the minimum control input $Z(G)$. The ratio of input nodes for SSC and total nodes in the graph is calculated by 
\begin{equation}
    \eta(G) = Z(G)/N
\end{equation}

We compare our reinforcement learning method with the algorithm based on degree (in Fig.~\ref{fig: alg1}) on different Erdos Renyi (ER) random graphs \cite{erdds1959random} with $N=50, 100, 150, 200, 250, 300, 350, 400, 450, 500$ and $p=0.04, 0.05, 0.1, 0.15, 0.2, 0.3$. $N$ is the number of nodes in the graph and $p$ is the probability of an edge exists between two nodes. The results are shown in Fig.~\ref{fig: compare}. The reinforcement learning method always has better results than the algorithm based on degree, except on the graph with $N=50, p=0.04$, where these two methods have the same results. In the first subfigure of Fig.~\ref{fig: random}, $\eta(G)$ increases with $N$ of ER graphs with the same connection probability $p$. $\eta(G)$ also increases with connection probability $p$ of ER graphs with the same $N$. $\eta(G)$ increases fast from $<k>=0$ to $<k>=50$ and approaches 1 at last which means that in a graph with high average degree, we need to control almost all nodes for SSC. That is to say, in a sparse graph, to control all nodes, we only need to control a small part of nodes with low in degree, but in a dense graph, we need to control almost all nodes to fully control the system. Here, we briefly explain why the average degree $<k>$ affects $\eta(G)$. For example, a node $v_i$ has $N$ out-neighbour nodes. To control all out-neighbour nodes of $v_i$, we need to control $N-1$ of them because node $v_i$ can not force one white out-neighbour node. In this situation, $\eta(G) = \frac{N}{N+1}$, which increase with $N$. Usually, the graph has a more complex structure, and $\eta(G)$ may be determined by many factors together e.g. short distance, clusters, betweenness, etc.

% === IV. Transistor Class-F inv Rectifier ========================================
% =================================================================================
\section{Conclusions} \label{sec:conclusions}
This paper introduces an innovative approach designed to minimize the number of control signals imposed on nodes, aiming to make the networked system strong structure controllable (SSC) for both zero/nonzero and zero/nonzero/arbitrary structure. Our primary achievement lies in the establishment of the condition for SSC, incorporating a specific color change rule, and converting the SSC problem into a graph coloring problem. Based on the color change rule and degree distribution of the graph, we designed a simple algorithm to minimize the number of input nodes for SSC. Subsequently, we modeled the color change process as a Markov decision process (MDP) and proposed the utilization of a reinforcement learning methodology, incorporating directed graph neural networks which encapsulate the graph information. This approach optimize the MDP, with the aim of minimizing the number of input nodes. To validate our method, we applied it to a social influence network, revealing that the control signals consistently avoid the high-degree nodes but select nodes with low in-degree as input nodes. Additionally, we compared our method with a degree-based algorithm, which we designed, in several Erdos Renyi (ER) random graphs. In every case, the reinforcement learning method presented in this paper outperformed the algorithm based on degree distribution. Simultaneously, the number of input nodes increases with the average degree of ER random graph. A general observation is that the sparse networks are easier to fully control for any admissible numerical realizations with zero/nonzero/arbitrary or zero/nonzero structure than the dense networks. As a future direction, we plan to explore the analytic relationship between the minimum number of input nodes and network topology, considering factors such as betweenness, clusters and page links.

\ifCLASSOPTIONcaptionsoff
  \newpage
\fi

\bibliographystyle{IEEEtran}
\bibliography{ref}

\vfill


\end{document}